\documentclass{article}

\usepackage{microtype}
\usepackage{graphicx}
\usepackage{subfigure}
\usepackage{booktabs} % for professional tables

\usepackage{hyperref}
\usepackage{amssymb}
\usepackage{amsmath}
\usepackage{amsthm}
\usepackage{framed}

\newtheorem*{assumption*}{\assumptionnumber}
\providecommand{\assumptionnumber}{}
\makeatletter
\newenvironment{assumption}[2]
 {%
  \renewcommand{\assumptionnumber}{Assumption #1-$\mathcal{#2}$}%
  \begin{assumption*}%
  \protected@edef\@currentlabel{#1-$\mathcal{#2}$}%
 }
 {%
  \end{assumption*}
 }
\makeatother

\newtheorem{theorem}{Theorem}[section]
\newtheorem{definition}{Definition}[section]
\newtheorem{lemma}{Lemma}[section]
\newtheorem{corollary}{Corollary}[section]

\newcommand{\uu}{\ensuremath{\mathbf u}}

\DeclareMathOperator*{\argmin}{arg\,min}

\def\FPhat{{\ensuremath{\hat{F^+}}}}
\def\FGhat{{\ensuremath{\hat{F_g}}}}

\newcommand{\K}{\mathcal{K}}

\newcommand{\ignore}[1]{}

\def\reals{{\mathbb R}}
\def\naturals{{\mathbb N}}

\def\bold0{\mathbf{0}}

\newcommand\E{\mbox{\bf E}}

\def\w{\mathbf{w}}

\def\mukappa{\mathbf{(\mu+\kappa)}}
\def\X{\mathbf{X}}
\def\T{\mathbf{T}}
\def\Xplus{\mathbf{X^+}}
\def\Fplus{F^+}
\def\Xdash{\mathbf{X^{'}}}
\def\TFplus{\mathbf{T_{F^+}}}
\def\metricD{\mathbb{D}}
\def\d{\mathbf{d}}

\newcommand{\B}{\mathbb{B}}
\newcommand{\Sp}{\mathbb{S}}

\newcommand{\eps}{\varepsilon}

\usepackage[accepted]{icml2020}

\icmltitlerunning{Data transformation insights in self-supervision with clustering tasks}

\begin{document}

\twocolumn[
\icmltitle{Data transformation insights in self-supervision with clustering tasks}




\begin{icmlauthorlist}
\icmlauthor{Abhimanu Kumar}{to}
\icmlauthor{Aniket Anand Deshmukh}{to}
\icmlauthor{Urun Dogan}{to}
\icmlauthor{Denis Charles}{to}
\icmlauthor{Eren Manavoglu}{to}%{to}
\end{icmlauthorlist}

\icmlaffiliation{to}{Bing Ads, Microsoft
,One Microsoft Way, Redmond, WA, USA}

% \icmlaffiliation{to}{Department of Computation, University of Torontoland, Torontoland, Canada}
% \icmlaffiliation{goo}{Googol ShallowMind, New London, Michigan, USA}
% \icmlaffiliation{ed}{School of Computation, University of Edenborrow, Edenborrow, United Kingdom}

\icmlcorrespondingauthor{Abhimanu Kumar}{abhimank@cs.cmu.edu}
% \icmlcorrespondingauthor{Eee Pppp}{ep@eden.co.uk}

\icmlkeywords{Machine Learning, ICML}

\vskip 0.3in
]



\printAffiliationsAndNotice{}  % leave blank if no need to mention equal contribution

\begin{abstract}
Self-supervision is key to extending use of deep learning for label scarce domains. For most of self-supervised approaches data transformations play an important role. However, up until now the impact of transformations have not been studied. Furthermore, different transformations may have different impact on the system. We provide novel insights into the use of data transformation in self-supervised tasks, specially pertaining to clustering. We show theoretically and empirically that certain set of transformations are helpful in convergence of self-supervised clustering. We also show the cases when the transformations are not helpful or in some cases even harmful. We show faster convergence rate with valid transformations for convex as well as certain family of non-convex objectives along with the proof of convergence to the original set of optima. We have synthetic as well as real world data experiments. Empirically our results conform with the theoretical insights provided.

\end{abstract}

% \section{IIC}
% \begin{figure}[ht]
% \vskip 0.2in
% \begin{center}
% \centerline{\includegraphics[width=\columnwidth]{IIC.PNG}}
% \caption{IIC}
% \label{icml-historical}
% \end{center}
% \vskip -0.2in
% \end{figure}
% \label{submission}

% Let \( \Phi:  \mathcal{R}^{d} \rightarrow \mathcal{Y} \) be the cluster probability function of a given point. Let \( \mathcal{I}: \mathcal{Y} \times \mathcal{Y} \rightarrow \mathcal{R}^{1} \) be the mutual information. Let \( a^{\prime} = g(a)\) be a transformed data point for a given transformation \( g:  \mathcal{R}^{d} \rightarrow \mathcal{R}^{d}\).  IIC solves following optimization problem using stochastic gradient descent: 
% \begin{eqnarray}
% \max_{\Phi} \sum_{i = 1}^n \mathcal{I} (\Phi(a_i), \Phi(a_i^{\prime}))
% \end{eqnarray}

% Learnt \( \Phi \) gives a soft cluster assignment for any data point. 

% \section{Convex Clustering}
% Convex clustering solves following optimization problem: 
% \begin{eqnarray}
% \min_{X \in \mathcal{R}^{d}} \sum_{i = 1}^n \|x_i - a_i \|^2 + \gamma \sum_{i = 1}^n \sum_{j < i}^n \| x_i - x_j \|_p
% \end{eqnarray}

% Two data points \( a_i \) and \( a_j \) are assigned same cluster if and only if \( x_i = x_j\). In this objective, one does not need to specify number of clusters ( \( \gamma \) indirectly decides the number of clusters). 

% \section{Goal}

\section{Introduction}
Labels are a scarce resource for prediction tasks especially when the labeling needs domain expertise to annotate data points. To address the label scarcity issue many machine learning practitioners have taken self-supervised learning~\cite{yengera2018surgical, Utsumi2019} approach. The idea is to create labels from the data points implicitly during the training phase of learning by learning robust feature representations~\cite{hendrycks2019using}. A common aspect of self-supervision is to add additional transformed data points to aid in the learning phase~\cite{misra2019selfsupervised, tian2019contrastive, ji2018invariant}. This has empirically shown to achieve high accuracy rates for a variety of machine learning tasks such as image recognition ~\cite{taha2018stream, Singh2018SelfSupervisedFL}, clustering~\cite{Caron_2018_ECCV}, classification~\cite{10.1145/1645953.1646072}, few-shot learning~\cite{Gidaris_2019_ICCV}, semi-supervised learning~\cite{rebuffi2019semisupervised}, learning to rank~\cite{pami/LiuWB19} etc. Though, this aspect of self-supervision has not been studied well theoretically or empirically. 

Self-supervision is a promising area of machine learning especially for prediction tasks where getting ground-truth labels is difficult or costly. In a self-supervised method, the learning approach is to get a robust representation learning of the feature set that also helps in implicit labeling~\cite{NIPS2019_9697}. The robust feature representation learnt through this approach can be used as an embedding or other intermediate feature for downstream tasks~\cite{halimi2018selfsupervised}. Data augmentation through transforming the original data points have been found to be helpful in such representation learning tasks~\cite{ji2018invariant, NIPS2017_7175}. There have been very few empirical works~\cite{pal2019hypothesis} and no theoretical work regarding the type of transforms that help in such representation learning schemes. This work is an attempt to fill this gap by providing empirical and theoretical analysis of data transformations applied to self-supervised representation learning for clustering tasks.

We take the clustering problem, solved via gradient descent schemes, as the main setup for our study. We show that certain set of transformations are helpful in convergence of the learning algorithm. Specifically, our contributions are: 
\begin{itemize}
    \item Show that under positive self-supervision, optima remain the same after adding transformed data points to the clustering problem.
    \item Provide empirical and theoretical insight into the rate of convergence for convex clustering loss augmented with transformed data.
    \item Show theoretical and empirical evidence of higher rate of convergence to the optima for a sub-family of non-convex clustering loss under graduated descent. The earlier guarantees have been probabilistic in nature~\cite{pmlr-v48-hazanb16} . 
    \item Provide insights into the varying degrees of contribution of data transforms during different phases of the learning scheme in a clustering task.
    \item Contrast examples of bad transforms that harm the learning process while helpful transforms that aid the learning process.
\end{itemize}
We provide empirical results over synthetic as well as real world datasets under both---convex and non-convex settings. For non-convex settings we use deep learning based clustering method. Our empirical results confirm the theoretical insights we obtain.

\section{Related Work}
\citet{abs-1902-09229} provide theoretical analysis of self-supervision in approaches similar to word2vec embeddings where words closer to target words are treated as ``positive" and words farther are treated ``negative". The implicit positive and negative example words can be understood to be the self-supervised labels.
\citet{pal2019hypothesis} provide empirical analysis of transforms in self-supervised image recognition tasks. They work with the hypothesis that transforms that produce data points that can not be predicted by current data points are helpful compared to transforms that produce data points that can be predicted using existing data points. They provide experimental results for CIFAR10, CIFAR100, FMNIST, and SVHN datasets.
\citet{iccv/0004HG17} propose transitive invariance representation learning by building a graph of visually variant image of same instance. They use a Triplet-Siamese network exploiting this visual-variance graph and apply the learned representations to different recognition tasks.

\citet{tian2019contrastive} take image samples from two different distributions or views, and learn representations by contrasting incogurent and congruent views. They use mutual information as the loss objective and maximize the mutual information between similar views. This methods provides good gains on ImageNet~\cite{imagenet_cvpr09} calssification accuracy as well as on STL10 dataset. 
\citet{ji2018invariant} provide a self-supervised classification approach by learning highly pure semantic clusters through maximizing the mutual information between a given image and its transform. They use common transforms such as cropping, rotation, shear etc. and show that maximizing mutual information between original image and its transform results in semantic clusters since the deep network is trained for robust representation learning. They show good gains on STL, CIFAR, MNIST, COCO-Stuff, and Potsdam datasets setting the state of the art for unsupervised clustering and segmentation.
\citet{taha2018stream} provide a scheme containing two parallel stacked deep learning architecture in a self-supervised setting for action recognition task. They feed the RGB frame to first tower--the spatial tower, and motion stream encoded using stack of differences into second tower--the motion tower. They show good gains on HMDB51, UCF101 and Honda Driving Dataset. 

\citet{NIPS2017_7175} propose image decomposition into underlying intrinsic images (or transforms) and combine them together using reconstruction loss to get robust intermediate representations. This self-supervised scheme is efficient in utilizing large scale unlabeled data for downstream knowledge transfer tasks. \citet{narayanan2018semisupervised} combine self-supervised scheme with small number of labels to get good accuracy for driver behavior detection task. 
DeepCluster~\cite{Caron_2018_ECCV} proposes a scheme where the initial labels for the clustering task come from a k-means run over the data. Then the network learns a feature representation for the clustering task using these labels. This self-supervised learning scheme extracts useful general purpose visual features resulting in better performance on ImageNet classification as well as standard transfer task.

\section{Background and Definitions}
Self-supervised representation learning in clustering is a popular task in machine learning~\cite{pmlr-v101-sun19a, abs-1905-00149, 10.1007/978-3-642-38562-9_26}. The clustering objective can be convex~\cite{NIPS2007_3181, ieee_SSP11} as well non-convex~\cite{ji2018invariant, NIPS2017_7175}. This section provides the mathematical setup for clustering loss and additional definitions to work with in case of non-convex loss.  
\subsection{Problem Setting}
\label{subsec:ProbSetting}
The machine learning task is to cluster the given set of data points $X \in \reals^{N\times d}$ into $k \in \naturals$ clusters under a given distance metric $\metricD$. $F: \reals^{N\times d} \rightarrow \reals$ is the clustering objective function which is to be optimized. $F$ may or may not be convex in the estimated parameters space $W \in \reals^d$. $X_{N\times d} \in \reals^{N\times d}$  are the data points where $d$ is the feature dimension and $N$ is the number of training samples. Given $X$ and $W$, equation~\ref{objctv} below is our objective function to be optimized. 
\begin{equation}
\label{objctv}
F(X,W) = \frac{1}{N}\sum^N_{i=1} f(X_i, W) = \frac{1}{N}\sum^N_{i=1} f_i(W)
\end{equation}
$f_i(W)$ indicates that when optimizing for $W$ each data point $i$ is an additive component of the objective. Following this, the gradient of objective w.r.t. $W$ is of additive form as expressed in equation~\ref{objGrad}. 
\begin{equation}
\label{objGrad}
\nabla_W F(W) = \frac{1}{N}\sum^N_{i=1} \nabla_W f_i(W)
\end{equation}
With addition of transformed data points into the data sample, the objective function takes the form of equation~\ref{objMdfd}, where $g: \reals^d \rightarrow \reals^d$ is the transform function and $N^{'}$ are the number of transformed data points. We define $\frac{1}{N^{'}+N}\sum^{N^{'}}_{j=1} f_{g_j}(W)$ as $F_g(W)$
\begin{eqnarray}
\label{objMdfd}
F^+(X,W) &=& \frac{\big(\sum^N_{i=1} f(X_i, W)+\sum^{N^{'}}_{j=1} f(g(X_j), W)\big)}{N+N^{'}} %\nonumber\\
%&+& \frac{1}{N+N^{'}}
\nonumber\\ 
        &=& \frac{1}{N+N^{'}}\big(\sum^N_{i=1} f_i(W) 
+ \sum^{N^{'}}_{j=1} f_{g_j}(W)\big) \nonumber\\ 
        &=& F(W) + F_g(W)
\end{eqnarray}
Equation~\ref{objMdfdGrad} gives the gradient updates for $F^{+}$. \begin{eqnarray}
\label{objMdfdGrad}
\nabla_W F^+(W) &=& \frac{\big(\sum^N_{i=1} \nabla_W f_i(W) 
+ \sum^{N^{'}}_{j=1} \nabla_W f_{g_j}(W)\big)}{N+N^{'}} \nonumber\\
    &=& \nabla_W F(W) + \nabla_W F_g(W)
\end{eqnarray}
For the clustering task, let $C_i, i \in \{1\ldots k\}$ are the ground truth cluster centroids w.r.t. distance metric $\metricD$. In typical self-supervised clustering schemes~\cite{pmlr-v101-sun19a, abs-1905-00149, 10.1007/978-3-642-38562-9_26} adding transformed data points results in objective of form equation~\ref{objMdfd}.
\subsection{Definitions and Notations}
The following definitions and notations will be primarily used for the non-convex clustering loss setting. The notations and definitions follow~\citet{pmlr-v48-hazanb16}, but are modified appropriately to suit our  self-supervised representation learning setting.
\paragraph{Notation:} We use $\B,\Sp$ to denote the
unit Euclidean ball/sphere in $\reals^d$, and also $\B_r(\w),\Sp_r(\w)$ as the Euclidean $r$-ball/sphere in $\reals^d$ centered at $\w$. For a set $A\subset \reals^d$ , 
$\uu \sim A$ denotes a random variable distributed uniformly over $A$.
\begin{definition}[$\delta$-Smoothed Function]
\label{fpsmoothDef}
Assuming any function $f$ with lipschitz smooth gradient  $||\nabla f(u) - \nabla f(v)||  \leq L||u - v|| ~~\forall u,v \in \reals^d$,  $\delta$-smooth version of $f$ is defined as
$$ \hat{f}_\delta(\w) = \E_{\uu \sim \B } [f(\w + \delta \uu) ]  .$$
\end{definition}
\begin{definition}[GradOp: $\delta$-Smoothed Gradient Operator]
\label{fig:SGO_G}
Given a function $f$, $\w\in \reals^d$, and smoothing parameter $\delta$  $\delta$-Smoothed Gradient Operator for $f$ at $\w$ is defined as
\begin{equation}
\text{GradOp}(f, \w, \delta) = \E_{\uu \sim \B }[\nabla f(\w+\delta \uu)], ~~\uu\sim\B
\end{equation}
\end{definition}
% \begin{figure}%[t]
% \begin{framed}
%  \textbf{Gradient Oracle }:  $\text{SGO}_G$\\
%  \textbf{Input}:  $\w\in \reals^d$, smoothing parameter $\delta$\\
% \textbf{Return}: $\E_{\uu \sim \B }[\nabla f(\w+\delta \uu)]$, where $\uu\sim\B$
% \end{framed}
% \caption{ Smoothed gradient oracle over unit ball $\B$.}
% \label{fig:SGO_G}
% \end{figure}

% Let \( A \in \mathcal{R}^{d \times n} = [a_1, a_2, \dots, a_n] \) be a given data matrix. Goal is to cluster \( n \) data  points into \( c \) number of clusters.

\section{Data Transformation -- Analytical Insights}
Building on the problem setting in the previous section, we analyze the clustering loss with added transformed data points in self-supervised representation learning. The analysis broadly consists of two parts: 1) convex clustering loss setting, and 2) non-convex clustering loss setting. Each of these subsections have specific set of assumptions independent of the other. We first discuss the global assumptions and theorems that are valid for both convex and non-convex settings. 
\begin{assumption}{1}{A}\label{asmptn1a}
\textbf{Positive Self Supervision:} As defined in the section~\ref{subsec:ProbSetting} given distance metric $\metricD$, and  ground truth cluster centroids $C_i, i \in \{1\ldots k\}$ :
\begin{equation}
\label{asmptn1aEqn}
X_i \in C_j \implies g(X_i) \in C_j \forall i,j
\end{equation}
\end{assumption}
Informally, assumption~\ref{asmptn1a} implies that the transformed data point maintains the same class/cluster membership as the original data point.
\begin{assumption}{1}{B}\label{asmptn1b}
\textbf{Unique Global Optima:} If $u = \argmin_W F(W)$ and $v = \argmin_W F(W)$ then
\begin{equation}
\label{asmptn1bEqn}
u = v
\end{equation}
\end{assumption}
Assumption~\ref{asmptn1b} implies that there is only one set of parameters $W^*$ that is the global optima.
\begin{theorem}[Unchanged Optima]
\label{thmUnchngdOptm}
Given assumptions~\ref{asmptn1a}, ~\ref{asmptn1b}, $C_i, i \in \{1\ldots k\}$ are cluster centroids, and  $W^*$ and $W^*_+$ are optima for objectives $F(W)$ and $F^+(W)$ respectively then
\begin{equation*}
W^* = W^*_+
\end{equation*}
\end{theorem}
Theorem~\ref{thmUnchngdOptm} above holds for all positive self-supervised representation learning cases with clustering tasks under assumptions~\ref{asmptn1a} and ~\ref{asmptn1b}. The next subsection provides guarantees for learning under convex loss.
\subsection{Convex Loss Settings}
Under convex loss, self-supervised representation learning for clustering tasks has faster linear rate of convergence when augmenting the learning scheme with transformed data points. To arrive at this insight we make the following assumptions. 
\begin{assumption}{2}{A}\label{asmptn2a}
\textbf{Strong Convexity:} 
\begin{eqnarray}
\label{asmptn2aEq}
F(u)\geq F(v) + <\nabla F(v),u-v> +\frac{\mu}{2}||u-v||^2_2, \nonumber\\ \forall u,v \in \reals^d
\end{eqnarray}
for some $\mu > 0$
\end{assumption}
\begin{assumption}{2}{B}\label{asmptn2b}
\textbf{Lipschitz Smooth Gradient:} 
\begin{eqnarray}
\label{asmptn2bEq}
||\nabla F^+(u) - \nabla F^+(v)||  \leq L||u - v|| \forall u,v \in \reals^d
\end{eqnarray}
for some $L > 0$
\end{assumption}
\begin{assumption}{2}{C}\label{asmptn2c}
\textbf{Strongly Convex Transform:}
\begin{eqnarray}
\label{asmptn2cEq}
F_g(u)\geq F_g(v) + <\nabla F_g(v),u-v> + \frac{\kappa}{2}||u-v||^2_2, \nonumber\\ \forall u,v \in \reals^d
\end{eqnarray}
for some $\kappa > \mu$.
\end{assumption}
Informally, assumption~\ref{asmptn2c} implies that the transformed data points are such that the objective is more strongly convex in $W$ than earlier data points. Lemma~\ref{lem:FPstrong} below is used in the proof for Theorem~\ref{thmCnvxRt}.
\begin{lemma} \label{lem:FPstrong}
Given assumptions ~\ref{asmptn2a} and ~\ref{asmptn2c},  $\Fplus$ is $\mukappa$-strongly-convex.
\end{lemma}
\begin{theorem}[Faster Convergence]
\label{thmCnvxRt}
Let assumptions~\ref{asmptn2a}, \ref{asmptn2b}, and \ref{asmptn2c} are true.  Given initial parameter $W_0 \in \mathbb{R}^d$ and $\frac{1}{L}\geq \eta > 0$, the iterates 
$$W^{t+1} = W^t - \eta\nabla F^+(W^t),$$ 
converge according to 
\begin{equation}
||W^{t+1}-W^*||^2_2 \leq (1-\eta(\mu+\kappa))^{t+1}||W_0-W^*||^2_2 
\end{equation}
i.e. $\eta =\frac{1}{L}$ the iterates enjoy a linear convergence with a rate of $(\kappa + \mu)/L$ instead of $\mu/L$ for non-transformed objective $F$.
\end{theorem}
Algorithm~\ref{Alg:SuffixGD} is a typical gradient descent algorithm for  convex loss.
\begin{algorithm}%[h!]
\caption{GradDescent} 
    \begin{algorithmic}%[1]
    \STATE  \textbf{Input}: data points $\X$, total time $\T$, decision set $\K$,  initial point $\w_1 \in \K$, gradient oracle $\text{GradOracle}(\cdot)$
    \STATE set $\eta$ as $\frac{1}{L}\geq \eta > 0$ 
    
    	\FOR { $t=1$ to $\T$} 
	   %  \STATE Set $\eta_t = 1/\mukappa t$ 
	      \STATE Query the gradient oracle at $\w_t$: 
	                   $$g_t \gets \text{GradOracle}(\w_t) $$
	      \STATE Update: 
	      $\w_{t+1} \gets (\w_t-\eta g_t) $
	   %   $\w_{t+1} \gets \Pi_{\K}(\w_t-\eta g_t) $
	           %   $\w_{t+1} \gets \Pi_{\K}(\w_t-\eta_t g_t) $
	   
	\ENDFOR
     \STATE \textbf{Return}: $\w_{\T}$ %$\w_{\TFplus}: = \frac{2}{ \TFplus}\big(\w_{ \TFplus/2+1}+\ldots+\w_{\TFplus} \big)$
    \end{algorithmic}
    \label{Alg:SuffixGD}
   \end{algorithm}
Informally, theorem~\ref{thmCnvxRt} states that adding data transforms that introduce strong convexity lead to faster convergence of algorithm~\ref{Alg:SuffixGD} in positive self-supervised representation learning settings for clustering.
\subsection{Strongly-convex Transform Example}
We take the well studied convex clustering loss~\cite{ieee_SSP11, ICML-2011-HockingVBJ, abs-1810-02677} and use it to illustrate a real world transform that makes the objective more strongly-convex. The loss objective, $\Phi$ is: 
\begin{equation}
\Phi = \min\limits_{w \in \reals^{d\times n}} \frac{1}{2}\big( \sum_{i=1}^{n} \|w_i-x_i\|^2_2 + \gamma\sum_{i,j}\alpha_{i,j} \|w_i-w_j\|_2^2 \big)
\end{equation}
$w_i \in \reals^d$ are the cluster centroids and $x_i\in\reals^d$ are the data. $\alpha_{i,j}\in \reals$ and $\gamma\in\reals$ are the hyperparameters. We show that the new transformed data points' function component has stronger convexity than the original data points' component (equation ~\ref{objMdfd}), and from lemma~\ref{lem:FPstrong} we have the desired stronger-convexity guarantee on the objective. If the hessian of the loss ($\nabla^2\Phi$) has all its eigenvalues ($\lambda$s) bigger than the earlier loss then the objective is more strongly convex~\cite{byod_cvx_book}. $\nabla^2_{i,j} = \partial \phi^2 / \partial w_iw_j = \gamma \sum_{j} \alpha_{i,j}(w_i-w_j)$, and $\nabla^2_{i,i} = \partial \phi^2 / \partial w_i^2 = 1 + \gamma \sum_{j,j\neq i} \alpha_{i,j}$. The hessian can be rewritten as 
\begin{equation}
\label{eqn:hessian}
\nabla^2 = I + A - B
\end{equation}
where  $I$ is identity, $A$ is a diagonal matrix with $A_{i,i} = \gamma\sum_j\alpha_{i,j}$, and $B_{i,j}=\alpha_{i,j}$.
To construct a valid transform let hyperparameter $\alpha$ can have only two values ${\alpha_1, \alpha_2}$ where $\alpha_2 > \alpha_1$. For any existing data point pairs $(x_i, x_j)$, assume $\alpha_{i,j} = \alpha_1$. We choose the transform $g: \reals^d \rightarrow \reals^d$ such that it maps $x_i \rightarrow x_k$ such that $\alpha_{i,k} = \alpha_2$ and for any two transformed pair $(x_k,x_l)$ $\alpha_{k,l} = \alpha_2$. Using equation~\ref{eqn:hessian} we can show that $\big(\nabla^2\Phi_g\big)_{\lambda} \geq \big(\nabla^2\Phi\big)_{\lambda}$ i.e. eigenvalues of $\nabla^2\Phi_g$ are greater than eigenvalues of $\nabla^2\Phi$. This gives us a valid transform.

\subsection{Non-Convex Loss Settings}
For non-convex setting, we only analyze those problem families that have a unique global optima. They can have multiple local optima but not multiple global optima which is formally captured by assumption~\ref{asmptn1b}. The other assumption for this case are listed below.
\begin{assumption}{3}{A}\label{asmptn3a}
\textbf{Lipschitz Smooth Gradient:} 
This is same as assumption~\ref{asmptn2b}, though here function $\Fplus$ and $F$ are non-convex.
$$||\nabla F^+(u) - \nabla F^+(v)||  \leq L||u - v|| ~~\forall u,v \in \mathbb{R^d} $$
\end{assumption}

\begin{assumption}{3}{B}\label{asmptn3b}
\textbf{Strongly-Convex Nice Function with Positive Self-Supervision:} 
We borrow the definition of $\sigma$-Nice functions from ~\cite{pmlr-v48-hazanb16} and extend it to local positive self-supervision settings. We assume that :
\begin{enumerate}
\item  \textbf{Centering property:}  For every $\delta> 0 $, and every $\w^*_\delta \in \argmin_{\w\in\K} \FPhat_{\delta}(\w)$, there exists $ \w^*_{\delta/2} \in \arg\min _{x\in\K}\FPhat_{\delta/2}(\w)$, such that:
$$\|\w^*_\delta - \w^*_{\delta/2}\| \leq \frac{\delta}{2}$$ 
\item  \textbf{Local $\mu$-strong convexity of $\hat{F}$:} For every $\delta> 0 $ let $r_\delta=3\delta$, and $\w^*_\delta = \argmin_{\w\in\K} \hat{F}_{\delta}(\w)$,
  then over $B_{  r_\delta} (\w^*_\delta)$, the function $\hat{F}_{\delta}(\w)$
  is $\mu$-strongly-convex.
  
  \item  \textbf{Local $\kappa$-strong convexity of data-transformed component ($\FGhat$):} For every $\delta> 0 $ let $r_\delta=3\delta$, and  $\w^*_\delta = \argmin_{\w\in\K} \hat{F}_{\delta}(\w)$,
  then over $B_{  r_\delta} (\w^*_\delta)$, the function $\hat{F}_{g_{\delta}}(\w)$
  is $\kappa$-strongly-convex.
  \item \textbf{Locally Positive Self-Supervision:} data-transformed component $\hat{F}_{g_{\delta}}$ obeys assumption~\ref{asmptn1a}.
\end{enumerate}
\end{assumption}

\begin{algorithm}[t]
\caption{$\text{GradOptTransformed}$  }
    \begin{algorithmic}%[1]
    \STATE \textbf{Input}:  input data $\X$, decision set $\K$
    \STATE Get the transformed data points $\Xdash$
    \STATE $\Xplus$ = $\{\X, \Xdash\}$
    \STATE  Choose $\w_1 \in \K $ uniformly at random. 
     \STATE  Set $\delta_1 =  \textrm{diam}(\K)$
    	\FOR { $m=1$ to $M$} 
	\STATE // Perform GradDescent over $\FPhat_{\delta_m}$
	\STATE Set shrinked decision set,
	$$\K_{m} : = \K\cap B(\w_m,1.5\delta_m)$$ 
	\STATE Set $$\TFplus = \frac{2\ln{(\delta_m/(4 \textrm{diam}(\K_m))}}{\ln{(1-\eta\mukappa)}} $$
	\STATE Set gradient oracle for $\FPhat_{\delta_m}$,
	 $$\text{GradOracle}(\cdot) =\text{GradOp}(\cdot,\delta_m)$$
	\STATE Update:
	\begin{eqnarray*}
	\w_{m+1} \gets \text{GradDescent}(\Xplus, \TFplus ,\K_m, \w_{m},\\ \text{GradOracle} )
	\end{eqnarray*}
	 \STATE $\delta_{m+1} = \delta_m/2$
    \ENDFOR
    \STATE \textbf{Return}: $\w_{M+1}$
    \end{algorithmic}
   \label{alg:generic}
   \end{algorithm}
%%%%%%%%%

\begin{figure}%[htb!]
% \vskip 0.2in
\begin{center}
\centerline{\includegraphics[width=0.8\columnwidth]{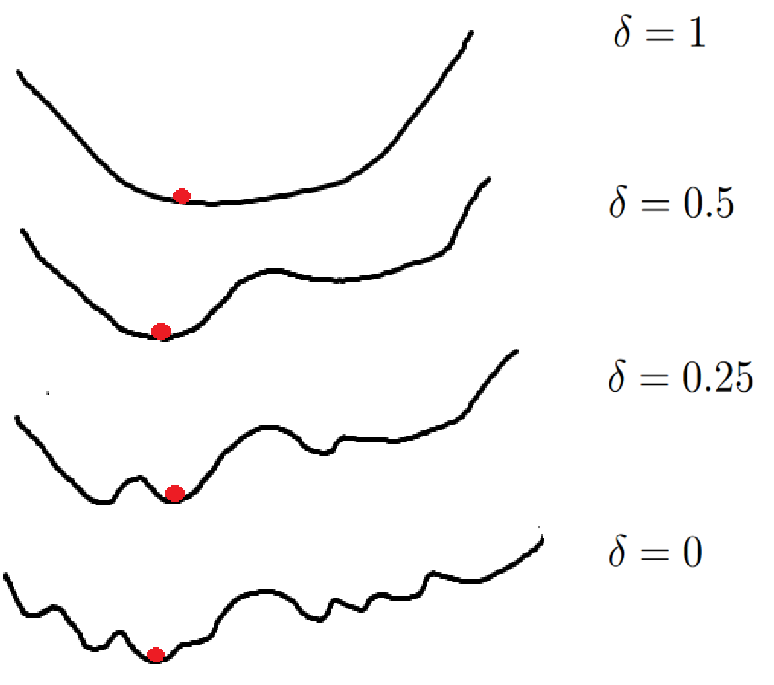}}
\caption{The function graph at different phases of algorithm~\ref{alg:generic} corresponding to different values of $\delta$. Higher $\delta$ leads to more smoothing. The red ball denotes the optima discovered in each outer iteration of algorithm~\ref{alg:generic}.}
\label{fig:garduatedEx}
\end{center}
% \vskip -0.2in
\end{figure}

\begin{lemma} \label{lem:gradopUnbias}
Assuming $\w\in \reals^d$ , $\delta \geq 0$, and $f$ is $L$-Lipschitz, $\text{GradOp}$ as defined in~\ref{fig:SGO_G} is an unbiased estimate for $\nabla\hat{f}_{\delta}(\w)$.
\end{lemma}
Lemma~\ref{lem:gradopUnbias} informally implies that for all practical purposes of the proofs later, $\text{GradOp}$ can be treated as a common gradient operation.

\begin{lemma} \label{lem:FPlocalComponents}
$$\FPhat_{\delta}(\w) = \hat{F}_{\delta}(\w) + \hat{F}_{g_{\delta}}(\w)$$
\end{lemma}
Lemma~\ref{lem:FPlocalComponents} implies that the locally smoothed $\FPhat_{\delta}$ can be broken into original and transformed components similar to its non-local counterpart (see equation~\ref{objMdfd}).
\begin{lemma} \label{lem:FPlocalStrong}
Given assumptions ~\ref{asmptn1b} and ~\ref{asmptn3b}, and for every $\delta> 0 $ let $r_\delta=3\delta$, and $\w^*_\delta = \argmin_{\w\in\K} \FPhat_{\delta}(\w)$,
  then over $B_{  r_\delta} (\w^*_\delta)$, $\FPhat$ is $\mukappa$-strongly-convex
\end{lemma}

\begin{corollary}
\label{cor:CnvxRt}
Theorem~\ref{thmCnvxRt} combined with lemma~\ref{lem:gradopUnbias} implies that in algorithm~\ref{alg:generic} for $$ \TFplus \geq \frac{2\ln{(\delta_m/(4 \textrm{diam}(\K_m))}}{\ln{(1-\eta\mukappa)}} $$
$$||\w_{m+1} - \w^*_m|| \leq \delta_{m+1}/2$$
\end{corollary}
Corollary~\ref{cor:CnvxRt} states that each internal call to GradDescent in algorithm~\ref{alg:generic} takes $\TFplus$ iterations. This fact will be used in the proof of main result for non-convexity, theorem~\ref{thnvmNncnvxPrbblty}.

\begin{lemma} \label{lem:induction}
Consider $M$, $\K_m$ and  $\w_{m+1}$ as defined in Algorithm~\ref{alg:generic}. Also denote by $\w^*_m$ the minimizer of $\FPhat_{\delta_m}$ in $\K_m$. Then the following  holds for all  $1\leq m \leq M$ :
\begin{enumerate}
\item The smoothed version $\FPhat_{\delta_m} $is $\mukappa$-strongly convex over $\K_m$, and $\w_{m}^* \in \K_m$.
\item And,  $||\w_{m+1} - \w_m^*|| \leq  \delta_{m+1}/2$
\end{enumerate}
\end{lemma}
Lemma~\ref{lem:induction} sates that in each outer iteration of algorithm~\ref{alg:generic} the local optima obtained are within $\delta_m/2$ distance of each other. 

\begin{theorem}[Non-Convex Convergence]
\label{thnvmNncnvxPrbblty}
Let  $\eps \in (0,1)$ and $\K$ be a convex set. Applying Algorithm~\ref{alg:generic},  
 after  %$\tilde{O}(\frac{\ln{(1/\eps)}}{\ln{(1/(1-\eta(\mukappa)))}})$ 
 $T_{total}$ rounds of total internal gradient updates,  Algorithm~\ref{alg:generic} outputs a point $\w_{M+1}$ such that
 
 $$||\w_{M+1} - \w^*|| \leq \delta_1(1-\eta\mukappa)^{\frac{T_{total}}{2(\ln{6}/\ln{1.5})}}$$
 
where $\w^*$ is the global optima for $\Fplus$.
\end{theorem}
The theorem above implies that algorithm~\ref{alg:generic} converges with a linear rate of $\mukappa$  instead of $\mu$ 
when the transformed data points are added to the positive self-supervised learning scheme. 

Figure~\ref{fig:garduatedEx} shows example runs of algorithm~\ref{alg:generic} over an arbitrary function. The graphs of the function is plotted at different smoothing stages. As the $\delta$ values decrease the function gets less smoother and more bumpier revealing more local optima. A point to note is that the change in the optima from one $\delta$ to the next is within range $\delta/2$ (assumption~\ref{asmptn3b}). Hence as seen in the figure~\ref{fig:garduatedEx}, discovered optima (red ball) in each phase doesn't change by much. 

\subsection{Data Transform Effect on Learning Stages}
\begin{corollary}
\label{cor:PhaseEffect}
Theorem~\ref{thmCnvxRt} combined with lemma~\ref{lem:gradopUnbias} implies that in algorithm~\ref{alg:generic} the transform makes the objective $\Fplus$ converge much faster in the initial iterations of the learning scheme than the later i.e.
$$\|\w_{m+1}-\w_m^*\| \leq \|\w_m-\w_{m-1}^*\|$$
\end{corollary}
Corollary~\ref{cor:PhaseEffect} implies that the transformed data points make the algorithm~\ref{alg:generic} much more faster in the initial phases than the later. This means that we can use the original and transformed data points in the first $k$ outer iterations of the algorithm~\ref{alg:generic} in the training phase and then use only the original data points after that to save computation and still get a reasonable speedup. 

\section{Data and Experiments}
We setup empirical experiments to verify our theoretical analyses for both convex and non-convex loss. We use the setup from~\citet{NIPS2007_3181} for our convex setting\footnote{\url{https://github.com/drumichiro/convex-clustering}}. The objective loss is defined in equation~\ref{eqn:cvxLossEg} 
\begin{equation}
\label{eqn:cvxLossEg}
F = \frac{1}{n}\sum_{i=1}^n\log\bigg[\sum_{j=1}^k q_j \exp^{-\beta\d_\phi(x_i,x_j)}\bigg]
\end{equation}
where ${x_i} \in \reals^d$ are the data points, $q_j$ is the probability of $x_j$ being a centroid (cluster-center), and $k$ is the total number of clusters. 
Our empirical results are reported on the objective in equation~\ref{eqn:cvxLossEg} with synthetic data. We set $k=4$ (four clusters) and $d=2$ (two dimensional data points), and cluster centroid coordinates as $ \{ (10,20), (30,20), (20,10), (20,30)\} $.
We add zero-mean Gaussian noise to the original data points and add them back as a positive self-supervised transform to illustrate the faster rate of convergence. We change the variance of the Gaussian to illustrate the effects of various transformations (``noisiness" in this case). We keep the optima fixed for the two objectives $F$ and $\Fplus$ given the synthetic data. For various cluster settings we compare the number of iterations taken to converge by $F$ and $\Fplus$, while constraining the new optima to be close to the original optima by a pre-defined threshold.

For non-convex settings we use DeepCluster~\cite{conf/eccv/CaronBJD18}. The objective loss for the framework is
\begin{equation}
\label{eqn:deepcluster}
\min\limits_{\theta,w} \frac{1}{N}\sum_{n=1}^{N} \ell\big( g_w\big(f_\theta(x_n) \big), y_n\big)
\end{equation}
where $x_n$ are the data points, $y_n \in \{0,1\}^k$ corresponding labels with $k$ possible classes, $f_\theta$ is the deep learning (convnet) mapping with parameters $\theta$.  $g_w$ is the classifier with parameters $w$ that takes feature representation $f_\theta$ and predicts a class. $\ell$ is the final multinomial logistic loss. 
We take the opensource code for  DeepCluster~\footnote{\url{https://github.com/facebookresearch/deepcluster}} and run it over random-10 ImageNet (ImageNet-10) data~\cite{deng2012imagenet}. 
We add rotation as a transform to the original data points and add them back to show its convergence impact. In every run of the experiment, we randomly choose 10 classes from ImageNet's original 1000 classes and train for getting representation of each image using DeepCluster algorithm. A point to be noted is that the transformed data points are only added to the training of feature representation extractor, and not to the logistic regression classifier which is the last layer of the deep learning network. This makes our results much more helpful in showing that the transforms are mostly helping the feature representation phase, whose output is an input to the final logistic layer. Learnt representations are used to train a multi-class classification model. We run each experiment four times and show the best accuracy of the model %and standard deviation 
across these four runs. We use 1,300 images per class and for the transformed setup we add equally number of transformed images to the original data per class.

%Using these We compare the convergence rate of two objectives $\F$ and $\Fplus$ with various data-transform settings keeping a fixed target accuracy. 

%For expriments with deepcluster, we show empirical evidence of data transformations on MNIST \cite{lecun1998gradient}, random-10 ImageNet \cite{deng2012imagenet}. 
\section{Empirical Insights}
Our experimental results conform to the theoretical insights that we provide in the previous sections. We specifically show faster convergence for both convex and non-convex loss settings. Some data transforms are more help than others--in fact some of them harm the learning process as we will see later.   

\begin{figure}[htb!]
% \vskip 0.2in
\begin{center}
\centerline{\includegraphics[width=1.05\columnwidth]{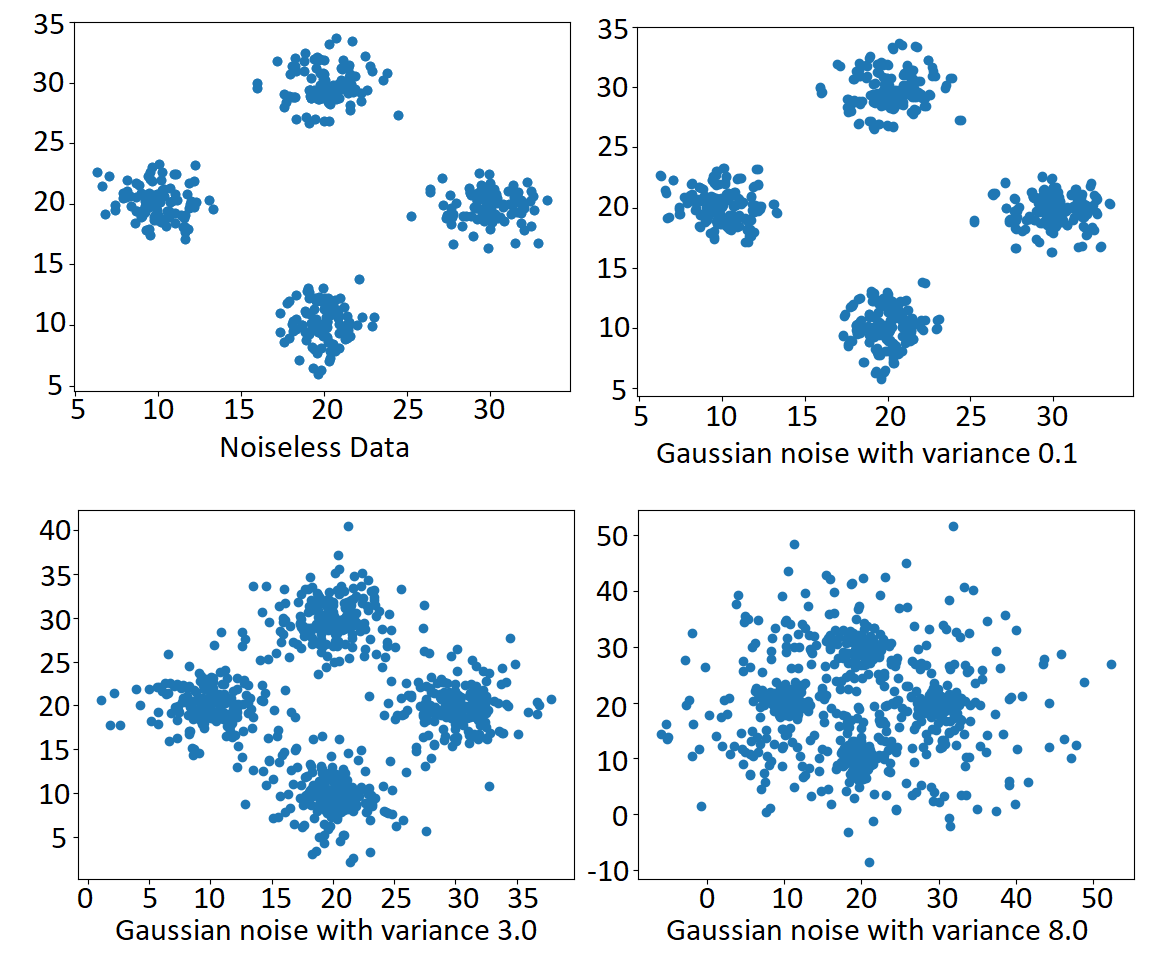}}
\caption{Datasets used for the convex clustering experimental setup. Going from left to right and top to bottom increases the noise variance. These are after adding the transformed data (noise) except the top left plot.}
\label{noise_0_0}
\end{center}
% \vskip -0.2in
\end{figure}

\begin{figure}[htb!]
% \vskip 0.2in
\begin{center}
\centerline{\includegraphics[width=1.05\columnwidth]{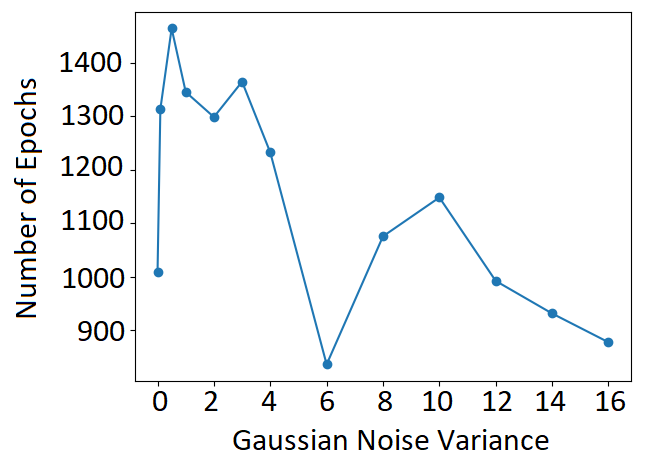}}
\caption{Number of epochs taken to converge to the global minima. X-axis represents amount of noise added and Y-axis represented number of epochs.}
\label{epochs_noise_avg}
\end{center}
% \vskip -0.2in
\end{figure}

\subsection{Convex Results}
Figure \ref{noise_0_0} top left shows the original 400 points generated with centroids $ \{ (10,20), (30,20), (20,10), (20,30)\} $. We add Gaussian noise to each data point with increasing variance. For each experiment there are a total of 800 points out of which 400 are original and 400 are noisy version. Figure \ref{noise_0_0} also shows data with noisy points added.  Moving left to right and top to bottom in figure~\ref{noise_0_0} adds more scatter to the plots.  To keep the convergence rate comparison fair we repeat each point twice so that original dataset also contains 800 points. As shown in figure \ref{epochs_noise_avg}, noisy versions converge faster compared to the original data. For example, clustering with original data takes around 1000 epochs to converge to global minima but clustering with noisy or transformed data points added at noise level 6 takes around 850 epochs. We also observe that not all transforms help in convergence. In figure~\ref{epochs_noise_avg} some transforms perform worse than the non-noise baseline, especially those with the noise variance between 0 and 6. This conforms to our analytical insight that strongly convex transforms are more helpful in the convergence.

\begin{figure}[htb!]
% \vskip 0.2in
\begin{center}
\centerline{\includegraphics[width=1.05\columnwidth]{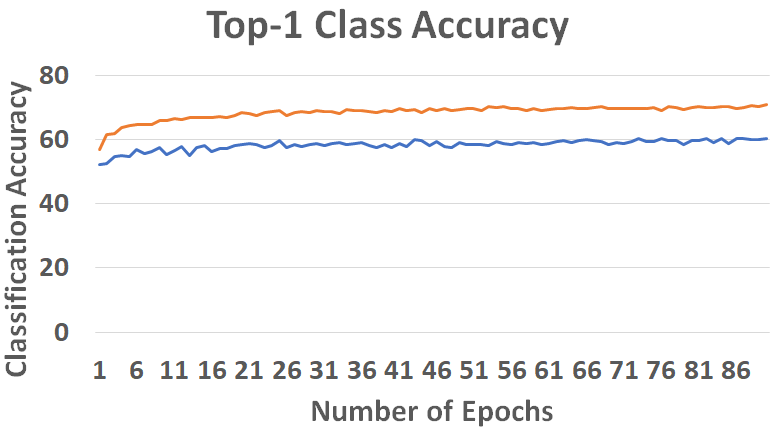}}
\caption{Number of epochs vs top-1 class classification accuracy on ImageNet-10. The orange line represents the accuracy with transformation and the blue line is accuracy with baseline.}
\label{fig:top1_accrcy}
\end{center}
% \vskip -0.2in
\end{figure}

\begin{figure}[htb!]
% \vskip 0.2in
\begin{center}
\centerline{\includegraphics[width=1.05\columnwidth]{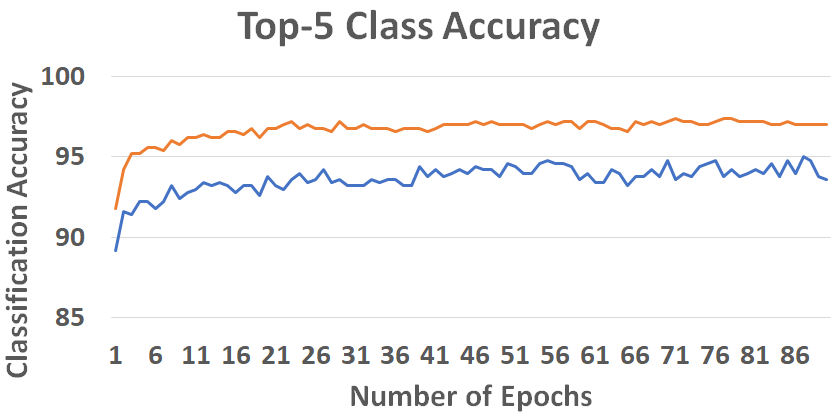}}
\caption{Number of epochs vs top-1 class classification accuracy on ImageNet-10. The orange line represents the accuracy with transformation and the blue line is accuracy with baseline.}
\label{fig:top5_accrcy}
\end{center}
% \vskip -0.2in
\end{figure}

\subsection{Non-Convex Results}
For non-convex case, our experimental results are based on the DeepCluster~\cite{conf/eccv/CaronBJD18} setup. We take image rotation as the transform. We add the transform only for the feature representation learning phase of DeepCluster. Given that our transformed setup outperforms the baseline non-transformed case, it is clear that the gains from transformation is mostly coming from feature representation learning. Figure~\ref{fig:top1_accrcy} shows the number of epochs vs top-1 class classification accuracy on ImageNet-10 dataset. We see that for each epoch the transformed setup consistently outperforms the non-transformed setup. Another notable aspect is that the gains are much superior in the early epochs for transformed setup compared to later epochs which conforms to our corollary~\ref{cor:PhaseEffect}. The transformed setup has a faster convergence rate for top-1 class classification as seen in table~\ref{tab:top1}. In fact though transformed setup has twice the number of data points, to reach the accuracy of 60\%, baseline setup takes 43 epochs whereas transformed setup takes 2 epochs i.e. it is 21.5 times faster. The baseline never reaches accuracies of 65\% or above for top-1 class classification.

\begin{table}[tbh!]
\label{tab:top1}
% \vskip 0.15in
\begin{center}
\begin{small}
\begin{sc}
\begin{tabular}{lccr}
\toprule
Accuracy & Baseline Epochs & Transformed Epochs  \\
\midrule
60\%  & 43 & 2 \\
65\%  & - & 9 \\
70\%  & -   & 53\\
\bottomrule
\end{tabular}
\end{sc}
\end{small}
\end{center}
% \vskip -0.1in
\caption{Number of epochs needed by transformed and baseline setup for different top-1 class classification accuracy levels for ImageNet-10. Baseline never reaches accuracies 65\% and above}
\end{table}

\begin{table}[tbh!]
\label{tab:top5}
% \vskip 0.15in
\begin{center}
\begin{small}
\begin{sc}
\begin{tabular}{lccr}
\toprule
Accuracy & Baseline Epochs & Transformed Epochs  \\
\midrule
91\%  & 2 & 1 \\
93\%  & 8 & 2 \\
95\%  & 87 & 3\\
97\%  & - & 22\\
\bottomrule
\end{tabular}
\end{sc}
\end{small}
\end{center}
% \vskip -0.1in
\caption{Number of epochs needed by transformed and baseline setup for different top-5 class classification accuracy levels for ImageNet-10. Baseline never reaches accuracies 97\% and above}
\end{table}

Figure~\ref{fig:top5_accrcy} shows top-5 class classification results. The transformed setup outperforms baseline setup consistently at each epoch. And the gains over baseline are achieved in early epochs as stated by corollary~\ref{cor:PhaseEffect}. Table~\ref{tab:top5} shows the number of epochs taken by transformed and baseline setups. Though transformed setup has twice the number of data points it is twice faster at reaching 91\% accuracy. It is four and twenty nine times faster at reaching accuracies of 93\% and 95\% respectively. Transformed setup reaches an accuracy of 97\% in 22 epochs whereas the baseline setup is unable to get there. This strongly suggests that valid transforms do help in reaching optima faster as reported in our analytical insights.

Overall the results imply that model which uses transformations consistently performs better than model that does not. The experiments conform to the theoretical insights obtained. It's a known fact that a well designed data augmentation or transformation significantly improves performance of deep learning~\cite{abs-1809-02492}. \citet{abs-1809-02492} observe that randomly pasting objects in the image hurts the performance of the learning task unless the placement is contextual. We show that one possible and probable explanation is that the transformed data points are making the objective more strongly convex in the parameter space, as hypothesized. This then makes the transformation context and data dependent---i.e. transformations will be helpful or harmful based on the data points. Precisely which transforms are helpful and how much do they make the objective strongly convex, however is beyond the scope of the paper given this is the first theoretical work in this subject matter, as far as we know. We will extend this work along this direction in the extended version of the paper.

\section{Conclusion}
The aim of this work is to analyse the effect of transforms in representation learning in the context of self-supervision. We choose clustering as our objective. Our basic insight---data transforms that make the objective strongly convex are useful---implicitly acknowledges that a transform is helpful or harmful depending on the data context.  \citet{abs-1809-02492} make similar observations in their work and propose an explicit context model to predict whether an image region is suitable for placing a given image object. We also observe in figure~\ref{epochs_noise_avg} that not all transformations are helpful and in the worst case harmful for the learning task. Our analysis also states that transforms are more helpful in the initial phases of the learning process. We observe this phenomenon empirically as well. Hence it is a good idea to augment the learning process with data transforms in the beginning and remove it later to save computation costs. 
Our work, as far as we know, is the first attempt in self-supervised literature to understand the impact of data transformation in predictive learning, theoretically or empirically. There have been some work in recent times in representation learning to analyze the effect of implicit positive and negative labels in self-supervised settings~\cite{abs-1902-09229}. But No works analyzing data transformation effects on self-supervision.

The empirical and theoretical insights presented here open interesting venues of research. Analyzing what other family of transforms besides strongly-convex are helpful in self-supervision is an important research question. We analyzed a subset of non-convex loss functions---those that obey $\sigma$-niceness. To extend this analysis to other non-convex objectives is worthwhile. Extending the data transformation effects from self-supervision to unsupervised or supervised approaches answers many interesting questions in machine learning.

%\section*{Acknowledgements}

%\textbf{Do not} include acknowledgements in the initial version of
%the paper submitted for blind review.

\bibliography{example_paper}
\bibliographystyle{icml2020}

\appendix
%\section{Do \emph{not} have an appendix here}
\onecolumn
\section{Theoretical Analysis}

% We next state the assumptions for the first theoretical result stated in theorem~\ref{thmUnchngdOptm}.

\subsection{Preliminaries}
\subsubsection{Proof of Theorem \ref{thmUnchngdOptm} in main paper}
\begin{proof}
The proof is by contradiction. Assume that $W^* \neq W^*_+$, in that case the centroids for objective $F$ and $F^+$ are different since $W^*$ and $W^*_+$ determine the respective optimal centroids and by assumption~\ref{asmptn1b} they are unique. But if centroids for objectives $F$ and $F^+$ are different, it contradicts assumption~\ref{asmptn1a}. Hence $W^* = W^*_+$.
\end{proof}
\subsection{Convex Proofs}

\subsubsection{Proof of Lemma \ref{lem:FPstrong} in main paper }
\begin{proof}
From equation~\ref{asmptn2aEq} and ~\ref{asmptn2cEq}
\begin{eqnarray}
\label{fPlsStrngCnvxty}
F(u)+F_g(u) &\geq& F(v) + F_g(v)  %\nonumber \\
+ \langle\nabla F(v),u-v\rangle + \langle\nabla F_g(v),u-v\rangle %\nonumber \\
+ \frac{\mu}{2}||u-v||^2_2 + \frac{\kappa}{2}||u-v||^2_2 \nonumber \\
\implies F^+(u) &\geq& F^+(v) + \langle\nabla F^+(v),u-v\rangle %\nonumber \\
+ \frac{\mu+\kappa}{2}||u-v||^2_2, \forall u,v \in \mathbb{R^d}
\end{eqnarray}
\end{proof}

\subsubsection{Proof of Theorem \ref{thmCnvxRt} in main paper}

\begin{proof}
Looking into the distance between Optima and $(t+1)_{th}$ update 
\begin{equation}
\left.\begin{aligned}
||W^{t+1}-W^*||^2_2 &= ||W^t - W^* - \eta \nabla F^+(W^t)||^2_2 \\
&= ||W^t - W^*||^2_2 - 2\eta\langle\nabla F^+(W^t), W^t-W^*\rangle + \eta^2||\nabla F^+(W^t)||^2_2 \\
&\leq (1-\eta(\mu+\kappa))||W^t - W^*||^2_2 - 2\eta(F^+(W^t) - F^+(W^*)) 
+ \eta^2||\nabla F^+(W^t)||^2_2  \\
&\mbox{above using equation \ref{fPlsStrngCnvxty}}  \\
&\leq (1-\eta(\mu+\kappa))||W^t - W^*||^2_2 - 2\eta(F^+(W^t) - F^+(W^*)) 
+ \eta^2L(F^+(W^t) - F^+(W^*))  \\
&\mbox{above using equation \ref{asmptn2bEq}}  \\
&= (1-\eta(\mu+\kappa))||W^t - W^*||^2_2 - 2\eta(1-L\eta)(F^+(W^t) - F^+(W^*))  \\
&\leq (1-\eta(\mu+\kappa))||W^t - W^*||^2_2\\
&\mbox{above using $\frac{1}{L}\geq \eta > 0$}  \\
\end{aligned}\right.
\end{equation}
Unrolling the above sequence $ (1-\eta(\mu+\kappa))||W^t - W^*||^2_2$ gives $||W^{t+1}-W^*||^2_2 \leq (1-\eta(\mu+\kappa))^{t+1}||W_0-W^*||^2_2$
\end{proof}
\subsection{Non-Convex Proofs}

% We provide the theoretical insights for the non-convex objectives next. We need to first define some notations. The notations and definitions are similar to ~\cite{pmlr-v48-hazanb16}.   
%  We state the assumptions and definitions next. 
\subsubsection{Proof of Lemma \ref{lem:gradopUnbias} in main paper}
\begin{proof}
Taking expectations on both sides and differentiating w.r.t. $\w$ gives us the desired result.
\end{proof}

\subsubsection{Proof of Lemma \ref{lem:FPlocalComponents} in main paper}
\begin{proof}
\begin{eqnarray}
\Fplus(\w) &=& F(\w) + F_g(\w) \nonumber \\
\implies \E_{\uu \sim \B } [\Fplus(\w + \delta \uu)] &=& \E_{\uu \sim \B } [F(\w + \delta \uu)] + \E_{\uu \sim \B } [F_g(\w + \delta \uu) ] \nonumber \\
\FPhat_{\delta}(\w) &=& \hat{F}_{\delta}(\w) + \hat{F}_{g_{\delta}}(\w) \nonumber 
\end{eqnarray}
\end{proof}

\begin{lemma} \label{lem:FPlocalStrong}
Given assumptions ~\ref{asmptn1b} and ~\ref{asmptn3b}, and for every $\delta> 0 $ let $r_\delta=3\delta$, and $\w^*_\delta = \argmin_{\w\in\K} \FPhat_{\delta}(\w)$,
  then over $B_{  r_\delta} (\w^*_\delta)$, $\FPhat$ is $\mukappa$-strongly-convex
\end{lemma}
\begin{proof}
From locally positive self-supervision, assumption~\ref{asmptn1a}, and ~\ref{asmptn3b} 
$$\argmin_{\w\in\K} \FPhat_{\delta}(\w) = \argmin_{\w\in\K} F_{\delta}(\w) $$
Using lemma~\ref{lem:FPlocalComponents} and the same arguments as in lemma~\ref{lem:FPstrong}, we conclude that $\FPhat$ is $\mukappa$-strongly-convex in $B_{  r_\delta} (\w^*_\delta)$
\end{proof}

\subsubsection{Proof of Theorem \ref{thnvmNncnvxPrbblty} in main paper}

\begin{lemma} \label{lem:induction}
Consider $M$, $\K_m$ and  $\w_{m+1}$ as defined in Algorithm~\ref{alg:generic}. Also denote by $\w^*_m$ the minimizer of $\FPhat_{\delta_m}$ in $\K_m$. Then the following  holds for all  $1\leq m \leq M$ :
\begin{enumerate}
\item The smoothed version $\FPhat_{\delta_m} $is $\mukappa$-strongly convex over $\K_m$, and $\w_{m}^* \in \K_m$.
\item And,  $||\w_{m+1} - \w_m^*|| \leq  \delta_{m+1}/2$
\end{enumerate}
\end{lemma}

\begin{proof}
The proof is by induction. Let us prove it holds for $m=1$. Note that $\delta_1 =  \textrm{diam}(\K)/2$, therefore $\K_1=\K$, and also $\w_1^*\in\K_1$.
 Also recall that $\mukappa$-niceness of $F^+$ implies that $\FPhat_{\delta_1}$ is $\mukappa$-strongly convex in $\K$, thus
by Corollary~\ref{cor:CnvxRt}, after less than $\TFplus = \frac{2\ln{(\delta_m/(4 \textrm{diam}(\K_1))}}{\ln{(1-\eta(\mukappa))}}$ optimization steps of GradDescent, we will have:
$$\| \w_2 - \w_1^* \| \leq \delta_{2}/2 $$
which establishes the case of $m=1$.
Now assume that lemma holds for $m>1$. By this assumption, $\| \w_{m+1} - \w_m^* \| \leq \delta_{m+1}/2$,  $\FPhat_{\delta_m}$ is $\mukappa$-strongly convex in $\K_m$, and also  $\w_m^* \in \K_m$. 
%Hence, we can use strong convexity of $\FPhat$  
%Equation~\eqref{eq:strongCvxity}  
%to get:
$$\|\w_{m+1} - \w_m^* \|\leq  \frac{\delta_{m+1}}{2}$$
Combining the latter with the centering property of $\mukappa$-niceness yields:
\begin{align*}
\|\w_{m+1} - \w_{m+1}^* \|&\leq \|\w_{m+1} - \w_m^* \|+\| \w_m^* -\w_{m+1}^*\| \\
& \leq 1.5\delta_{m+1} 
\end{align*}
and it follows that,
$$\w_{m+1}^* \in B(\w_{m+1},1.5\delta_{m+1})\subset B({\w}_{m+1}^*,3\delta_{m+1})$$
 Recalling that $\K_{m+1}: = B(\w_{m+1},1.5\delta_{m+1})$, and the local strong convexity property of $f$ (which is $\mukappa$-nice), then the induction step for first part of the lemma holds. Now, by Corollary~\ref{cor:CnvxRt}, after less than $\TFplus = \tilde{\mathcal{O}}(\frac{2\ln{(\delta_m/(4 \textrm{diam}(\K_m))}}{\ln{(1-\eta(\mukappa))}}) $ optimization steps of GradDescent over $\FPhat_{\delta_{m+1}}$, we will have:
$$\| \w_{m+2} - \w_{m+1}^* \| \leq \delta_{m+2}/2$$
which establishes the induction step for the second part of the lemma.
\end{proof}

% \begin{lemma}
% 	\label{lem:SmoothingLemma}
% 	Let $\FPhat_\delta$ be the $\delta$-smoothed version of $\Fplus$, then,
% 		 $$\forall \w \in \reals^d : |\FPhat _\delta (\w) - \Fplus(\w) | \le \delta L$$
% \end{lemma}

% \begin{proof} [Proof of Lemma \ref{lem:SmoothingLemma}]
% 	\begin{align*}
% 		|\FPhat_\delta (\w)-\Fplus(\w)|&=|\E_{u \sim \B } \left[  \Fplus(\w+ \delta \uu)\right] - \Fplus(\w) |  \\%&\mbox{definition of $\FPhat_\delta$}\\
% 		&\leq \E_{\uu \sim \B } \left[  |\Fplus(\w+ \delta \uu) - \Fplus(\w) |\right]   \\% &\mbox{linearity of expectancy and Jensen inequality} \\
% 		&\leq \E_{\uu \sim \B }\left[ L\| \delta \uu \| \right]   \\%& \mbox{$f$ is $L$-Lipschitz}\\
% 		&\leq L \delta %& \mbox{ $u\in \sphere$}
% 	\end{align*}
% in the first inequality we used Jensen's inequality, and in the last inequality we used $\|\uu\|\leq 1$ with Lipschitz smoothness, since $\uu\in \B$.
% \end{proof}

We are now ready to prove Theorem~\ref{thnvmNncnvxPrbblty}:
\begin{proof}[Proof of Theorem~\ref{thnvmNncnvxPrbblty}]
Algorithm~\ref{alg:generic} terminates after $M$ epochs meaning, $\delta_{M+1} =  \textrm{diam}(\K)/(1.5)^M$.
According to Lemma~\ref{lem:induction} the following holds every $\w\in \K$,
$$\|\w_{m+1} - \w_m^* \|\leq  \frac{\delta_{m+1}}{2}$$
We need to bound $\|\w_{M+1} - \w^* \|$

\begin{eqnarray}
\label{eqn:telescoping}
\|\w_{M+1} - \w^* \| &\leq& \|\w_{m+1} - \w_m^* \| + \|\w_{m+1}^* - \w_{m}^* \| + \|\w_{m+2}^* - \w_{m+1}^* \| + \ldots + \|\w_{\infty}^* - \w^* \| \nonumber \\
&\leq& \frac{\delta_{M+1}}{2} + \frac{\delta_{M+1}}{4} + \frac{\delta_{M+1}}{8} + \ldots \nonumber \\
&=& \frac{\delta_{M+1}}{2} \frac{1}{1 - \frac{1}{2}} \nonumber \\
&=& \delta_{M+1} = \frac{\delta_{1}}{(1.5)^M} = \eps \nonumber \\
\implies M &=& \frac{\ln{(\delta_{1}/\eps)}}{\ln{(1.5)}}
\end{eqnarray}

Now calculating $T_{total}$
\begin{eqnarray}
T_{total} &=& \sum_{m=1}^{M} T_{\Fplus_m} \nonumber \\
&=& \sum_{m=1}^{M} \frac{2\ln{(\delta_m/(4 \textrm{diam}(\K_m))}}{\ln{(1-\eta(\mukappa))}} \nonumber \\
&\leq& \sum_{m=1}^{M} \frac{2\ln{(\delta_m/(4(1.5)\delta_m)}}{\ln{(1-\eta(\mukappa))}} \nonumber \\
&=& \frac{2M\ln{(6)}}{\ln{(\frac{1}{1-\eta(\mukappa)})}} \nonumber \\
\implies T_{total} &\leq& \frac{2\ln{(\delta_{1}/\eps)}(\ln{(6)}/\ln{(1.5)})}{\ln{(\frac{1}{1-\eta(\mukappa)})}} \nonumber \\
&&\mbox{Above inequality using equation~\ref{eqn:telescoping}}. \nonumber \\
\implies T_{total} &\leq& \frac{2\ln{(\eps/\delta_{1})}(\ln{(6)}/\ln{(1.5)})}{\ln{(1-\eta(\mukappa))}} \nonumber \\
\implies \ln{\big(1-\eta(\mukappa)\big)} \frac{\ln{(1.5)T_{total}}}{2\ln{(6)}} &\geq& \ln{(\eps/\delta_{1})} \nonumber \\
&&\mbox{Above inequality by multiplying $\ln{\big(1-\eta(\mukappa)\big)}$, a negative quantity.} \nonumber \\
\implies \big(1-\eta(\mukappa)\big)^ {\frac{\ln{(1.5)T_{total}}}{2\ln{(6)}}} &\geq& \eps/\delta_{1} \nonumber \\
\implies \eps &\leq& \delta_{1}\big(1-\eta(\mukappa)\big)^ {\frac{\ln{(1.5)T_{total}}}{2\ln{(6)}}} \nonumber \\
||\w_{M+1} - \w^*|| &\leq& \delta_1(1-\eta\mukappa)^{\frac{T_{total}}{2(\ln{6}/\ln{1.5})}} \nonumber 
\end{eqnarray}

\end{proof}

%%%%%%%%%%%%%%%%%%%%%%%%%%%%%%%%%%%%%%%%%%%%%%%%%%%%%%%%%%%%%%%%%%%%%%%%%%%%%%%
%%%%%%%%%%%%%%%%%%%%%%%%%%%%%%%%%%%%%%%%%%%%%%%%%%%%%%%%%%%%%%%%%%%%%%%%%%%%%%%

\end{document}